\documentclass{article}
\usepackage{ijcai23}

\usepackage{times}
\usepackage{soul}
\usepackage{url}
\usepackage[hidelinks]{hyperref}
\usepackage[utf8]{inputenc}
\usepackage[small]{caption}
\usepackage{graphicx}
\usepackage{amsmath}
\usepackage{amsthm}
\usepackage{booktabs}
\usepackage{algorithm}
\usepackage[switch]{lineno}

\usepackage{amssymb}
\usepackage{mathrsfs}
\usepackage{mathtools}
\usepackage{bm}
\usepackage{bbm}
\usepackage{algpseudocode}
\usepackage{sgame}
\usepackage[capitalise]{cleveref}
\usepackage{subcaption}
\usepackage{apxproof}
\usepackage{booktabs}
\usepackage{centernot}

\DeclareMathOperator{\E}{\mathbb{E}}

\DeclareMathOperator{\pd}{\succ_\text{p}}
\DeclareMathOperator{\pde}{\succeq_\text{p}}
\DeclareMathOperator{\pdef}{\preceq_\text{p}}
\DeclareMathOperator{\fsd}{\succ_\text{FSD}}
\DeclareMathOperator{\fsde}{\succeq_\text{FSD}}

\DeclareMathOperator{\du}{\succ_\text{d}}

\DeclareMathOperator{\pf}{\text{PF}(\Pi)}
\DeclareMathOperator{\ch}{\text{CH}(\Pi)}
\DeclareMathOperator{\duset}{\text{DUS}(\Pi)}
\DeclareMathOperator{\cduset}{\text{CDUS}(\Pi)}

\DeclareMathOperator{\vpi}{\mathbf{V}^{\pi}}
\DeclareMathOperator{\vpip}{\mathbf{V}^{\pi'}}
\DeclareMathOperator{\zpi}{\mathbf{Z}^{\pi}}
\DeclareMathOperator{\zpip}{\mathbf{Z}^{\pi'}}

\newtheorem{theorem}{Theorem}[section]
\newtheorem*{theorem31}{Theorem 3.1}
\newtheorem*{theorem32}{Theorem 3.2}
\newtheorem*{theorem51}{Theorem 5.1}
\newtheorem{corollary}{Corollary}[theorem]
\newtheorem*{corollary411}{Corollary 4.1.1}
\newtheorem*{corollary511}{Corollary 5.1.1}
\newtheorem{proposition}{Proposition}[theorem]
\newtheorem{lemma}[theorem]{Lemma}
\newtheorem*{lemma41}{Lemma 4.1}
\theoremstyle{definition}
\newtheorem{definition}{Definition}[section]
\newtheorem{example}{Example}

\title{Distributional Multi-Objective Decision Making}

\author{
Willem Röpke$^1$\and
Conor F. Hayes$^2$\and
Patrick Mannion$^2$\and
Enda Howley$^{2}$\and
Ann Nowé$^1$\And
Diederik M. Roijers$^{1,3}$
\affiliations
$^1$Vrije Universiteit Brussel, Brussels, Belgium\\
$^2$University of Galway, Galway, Ireland\\
$^3$City of Amsterdam, Amsterdam, The Netherlands \\
\emails
willem.ropke@vub.be, c.hayes13@nuigalway.ie, patrick.mannion@universityofgalway.ie, enda.howley@universityofgalway.ie, ann.nowe@vub.be, diederik.roijers@vub.be
}

\begin{document}

\maketitle

\begin{abstract}
For effective decision support in scenarios with conflicting objectives, sets of potentially optimal solutions can be presented to the decision maker. We explore both what policies these sets should contain and how such sets can be computed efficiently. With this in mind, we take a distributional approach and introduce a novel dominance criterion relating return distributions of policies directly. Based on this criterion, we present the distributional undominated set and show that it contains optimal policies otherwise ignored by the Pareto front. In addition, we propose the convex distributional undominated set and prove that it comprises all policies that maximise expected utility for multivariate risk-averse decision makers. We propose a novel algorithm to learn the distributional undominated set and further contribute pruning operators to reduce the set to the convex distributional undominated set. Through experiments, we demonstrate the feasibility and effectiveness of these methods, making this a valuable new approach for decision support in real-world problems.
\end{abstract}

\section{Introduction}
Multi-objective sequential decision making is a complex process that involves trade-offs between multiple, often conflicting, objectives. As the preferences over these objectives are typically not known a priori, it is challenging to find a single optimal solution, and instead, a set of solutions that are considered optimal can be presented to the decision maker \cite{roijers2013survey}. To keep decision support tractable, it is necessary to reduce the size of the solution sets as much as possible. Therefore, defining appropriate solution sets that do not retain excess policies while guaranteeing that no concessions are made to optimality, as well as designing corresponding pruning algorithms is essential \cite{taboada2007practical}.

A solution set that is often considered appropriate in both multi-objective decision making and multi-objective optimisation is the Pareto front \cite{roijers2013survey}. The Pareto front consists of the policies that lead to Pareto optimal expected payoffs and thus contains all policies which are optimal for decision makers interested in optimising the utility from these expected returns \cite{hayes2022practical}. However, it is known that the Pareto front does not necessarily contain all optimal policies for problems where the decision maker optimises for their expected utility instead \cite{hayes2022expected}. 

To address this limitation, we introduce a novel dominance criterion, called distributional dominance, relating the multivariate return distribution between policies directly. Distributional dominance relies on first-order stochastic dominance, which is known to imply greater expected utility for univariate distributions \cite{fishburn1974convex,bawa1985determination}, and has also been explored for multi-variate distributions \cite{denuit2013multivariate,levy2016bivariate}. Based on distributional dominance, we propose the \emph{distributional undominated set (DUS)} as a novel solution set and show that it contains all optimal policies for the class of multivariate risk-averse decision makers defined by Richard~\shortcite{richard1975multivariate}. Furthermore, we show that it is a superset of the Pareto front and as a result is a suitable starting set which can be further pruned to smaller subsets for specific scenarios. 

While the DUS contains no distributionally dominated policies, it may still contain policies which will never be chosen in the expected utility setting. Therefore, we introduce a second solution set, the \emph{convex distributional undominated set (CDUS)}, which includes only those policies that are undominated by a mixture of policies in the DUS. We find that the CDUS is a subset of the DUS and contains all optimal policies for multivariate risk-averse decision makers. While in general the CDUS and the Pareto front do not coincide, both sets are shown to include the convex hull. 

From a computational perspective, we contribute algorithms to prune a set of policies to its DUS or CDUS. As these pruning methods rely on the quality of the input set, we present an extension of the Pareto Q-learning algorithm \cite{vanmoffaert2014multiobjective} to learn return distributions and only discard those policies that are not in the DUS. We evaluate our approach on randomly generated MOMDPs of different sizes and compare the sizes of the resulting sets after pruning. As our goal is to use these sets in a decision support scenario, keeping their sizes reasonable and algorithms tractable both in terms of runtime and memory enables decision makers to efficiently select their preferred policy\footnote{A full version with supplementary material is available online at \url{https://arxiv.org/abs/2305.05560}}.  

\section{Background}
\subsection{Multi-Objective Decision Making}
Sequential decision making is often formalised using Markov Decision Processes (MDPs) which provide a mathematical framework for modelling settings in which an agent must choose an action at each time step based on the current state of the system. To address real-world situations where decision makers must consider multiple conflicting objectives, MDPs can be generalised to Multi-Objective Markov Decision Processes (MOMDPs) which allow for vectorial reward functions \cite{roijers2017multiobjective}.

\begin{definition}
A multi-objective Markov decision process is a tuple $M = (\mathcal{S}, \mathcal{A}, T, \gamma, \mathbf{R})$, with $d \geq 1$ objectives, where:
\begin{itemize}
\item $\mathcal{S}$ is the state space;
\item $\mathcal{A}$ is the set of  actions
\item $T \colon \mathcal{S} \times \mathcal{A} \times \mathcal{S} \to \left[ 0, 1 \right]$ is the transition function;
\item $\gamma \in [0, 1]$ is the discount factor;
\item $\mathbf{R} \colon \mathcal{S} \times \mathcal{A} \times \mathcal{S} \to \mathbb{R}^d$ is the vectorial reward function.
\end{itemize} 
\end{definition}
In a MOMDP, a decision maker takes sequential actions by means of \emph{policy} $\pi: \mathcal{S} \times \mathcal{A} \to [0, 1]$ which maps state-action pairs to a probability. We denote the set of all policies by $\Pi$. 

We take a distributional approach \cite{bellemare2023distributional,hayes2022decision} and consider the multivariate return distributions of these policies. The return $\zpi = \left(Z_1^{\pi}, \dotsc, Z_d^{\pi}\right)^\text{T}$ is a random vector where each $Z_i^{\pi}$ is the marginal distribution of the $i$'th objective such that, 
\begin{equation}
\E \left[\zpi\right] = \E \left[\sum_{t=0}^\infty \gamma^t \mathbf{r}_t \mid \pi, \mu_0\right] = \left(\E \left[Z_1^{\pi} \right], \dotsc, \E \left[Z_d^{\pi}\right] \right)^\text{T}.
\end{equation}
For notational simplicity, when considering the expected returns directly we will write this as $\vpi = \left(V_1^{\pi}, \dotsc, V_d^{\pi}\right)^\text{T}$.

\subsection{Dominance Relations}
\label{sec:dominance_relations}
Multi-objective decision making presents additional complexity compared to traditional decision making, as it is not possible to completely order the return of different policies. Pareto dominance introduces a partial ordering by considering a vector dominant when it is greater or equal for all objectives and strictly greater for at least one objective. We say a policy Pareto dominates a second policy when the expected value of its return distribution is Pareto dominant. 
\begin{definition}
\label{def:pareto-dominance}
Let $\pi, \pi' \in \Pi$. Then $\pi$ Pareto dominates $\pi'$, denoted by $\vpi \pd \vpip$, when $\forall i, V_i^\pi \geq V_i^{\pi'} \land \exists i, V_i^\pi > V_i^{\pi'}.$
\end{definition}
When the expected return of $\pi$ is equal to $\pi'$ or Pareto dominates it, we denote this by $\vpi \pde \vpip$.

First-order stochastic dominance (FSD) is a well-known dominance criterion from decision theory and economics, which relates return distributions directly \cite{levy2016stochastic,denuit2013multivariate}. Let $F_\mathbf{X}(\mathbf{x}) = P(\mathbf{X} \pdef \mathbf{x})$ be the cumulative distribution function (CDF) of a random vector $\mathbf{X}$, denoting the probability that the random vector takes on a value Pareto dominated or equal to $\mathbf{x}$. Informally, we say that $\mathbf{X}$ FSD another distribution $\mathbf{Y}$ when it always has a higher probability of obtaining Pareto dominant returns. 
\begin{definition}
\label{def:fsd}
A policy $\pi$ first-order stochastically dominates another policy $\pi'$, denoted by $\zpi \fsde \zpip$, when,
\begin{equation*}
    \forall \mathbf{v} \in \mathbb{R}^d: F_{\zpi}(\mathbf{v}) \leq F_{\zpip}(\mathbf{v}).
\end{equation*}
\end{definition}

\subsection{The Utility-Based Approach}
\label{sec:utility_based_approach}
We take a utility-based approach to multi-objective decision making \cite{roijers2013survey} and assume that for any decision maker a utility function $u: \mathbb{R}^d \to \mathbb{R}$ exists that represents their preferences over the objectives. We consider the class of strictly monotonically increasing utility functions, denoted by $\mathcal{U}$. Intuitively, such utility functions imply that any decision maker prefers more of each objective, given all else equal.
\begin{definition}
\label{def:strictly-increasing}
A function $f: \mathbb{R}^d \to \mathbb{R}$ is called strictly monotonically increasing if,
\begin{equation*}
    \forall \mathbf{x}, \mathbf{y} \in \mathbb{R}^d: \mathbf{x} \pd \mathbf{y} \implies f(x) > f(y).
\end{equation*}
\end{definition}

In the utility-based approach, there is often a need to optimise for an entire class of decision makers or a decision maker for which we do not know the exact utility function. In this case, it is necessary to identify a set of policies that contain an optimal policy for all possible utility functions. A further complication arises from the fact that different optimality criteria exist depending on how the utility is derived \cite{roijers2013survey}. For scenarios where a decision maker's utility is derived from multiple executions of a policy, the scalarised expected returns (SER) criterion can be optimised,
\begin{equation}
    V_{u}^{\pi} = u\left(\mathbb{E} \left[ \sum\limits^\infty_{t=0} \gamma^t {\bf r}_t \:|\: \pi, \mu_0 \right]\right).
    \label{eqn:ser}
\end{equation}
Alternatively, it is possible that the decision maker only executes their policy once and therefore aims to optimise their expected utility. In the utility-based approach, this is known as the expected scalarised returns (ESR) criterion, 
\begin{equation}
\label{eqn:esr}
V_{u}^{\pi} = \mathbb{E} \left[ u\left( \sum\limits^\infty_{t=0} \gamma^t {\bf r}_t \right) \:|\: \pi, \mu_0 \right].
\end{equation}
It is well-established that, in general, optimal policies under one criterion need not be optimal under the other criterion \cite{roijers2013survey,vamplew2022impact}. 



\subsection{Solution Sets}
One of the most common solution sets in the literature is the Pareto front (PF), formally defined in \cref{def:pareto-front} \cite{roijers2017multiobjective}. We stress that this solution set is presented in the context of the SER criterion as it is based on the expected returns of the policies. 

\begin{definition}
\label{def:pareto-front}
The Pareto front is the set of all policies that are not Pareto dominated:
\begin{equation}
    \pf = \left\{ \pi \in \Pi \mid \nexists \pi' \in \Pi, \vpip \pd \vpi \right\}.
\end{equation}
\end{definition}

A second solution set that is often considered is the convex hull (CH) which contains all policies that are optimal under linear utility functions and is therefore applicable under both SER and ESR \cite{hayes2022practical}. Additionally, when stochastic policies are allowed, the convex hull can be used to construct all Pareto optimal policies \cite{vamplew2009constructing}.

\begin{definition}
\label{def:convex-hull}
The convex hull is the set of all policies that are not Pareto dominated by a convex combination of other policies,
\begin{equation}
    \ch = \left\{\pi \in \Pi \mid \nexists \lambda \in \Delta^{|\Pi|}: \sum_{i=1}^{|\Pi|} \lambda_i \mathbf{V}^{\pi_i} \pd \vpi\right\}.
\end{equation}
\end{definition}

We note that solution sets based on return distributions have also been considered, with for example the ESR set \cite{hayes2022expected}. In this work, we extend this line of research and provide additional theoretical and computational results.

\section{Distributional Decision Making}
While most of multi-objective decision making focuses on returning the Pareto front, we demonstrate that this does not cover the full range of optimal policies. Specifically, for decision makers optimising their expected utility, the best policy in the Pareto front may still be significantly worse than a Pareto dominated policy. To overcome this, we propose a novel dominance criterion and subsequently construct a solution set based on this criterion.

\subsection{Motivation}
To understand why it is necessary to construct these novel solution sets, and in particular why a distributional approach is appropriate, it is helpful to consider a motivating example.

\begin{example}
\label{exmp:utility}
Imagine a hospital patient needing to decide on a treatment plan with their doctor. Their objectives are to maximise the efficacy of the treatment, denoted $v_1$, while also maximising their comfort (i.e. minimise the side-effects), denoted $v_2$. Unfortunately, these objectives are conflicting. In previous discussions with their doctor, the patient mentioned that they wish to strike a balance between the two. A fitting utility function is the product between the two objectives (\cref{eq:utility-function}) as it is maximised when values are closer together.
\begin{equation}
\label{eq:utility-function}
u(v_1, v_2) = v_1 \cdot v_2
\end{equation}
The doctor then proposes the following two treatment plans. 
\begin{equation*}
\begin{split}
    A & = \left\{P(v_1=1, v_2=0) = \frac{1}{2}, P(v_1=0, v_2=1) = \frac{1}{2}\right\}\\
    B & = \left\{P(v_1=0.45, v_2=0.45) = 1\right\},\\
\end{split}
\end{equation*}
with $\E [A] = (0.5, 0.5)$ and $\E [B] = (0.45, 0.45)$.

When taking the standard approach and applying Pareto dominance, it is clear that the expected return of $A$ dominates that of $B$. In contrast, when considering the distributions on the basis of expected utility, $A$ has an expected utility of 0, while $B$ has an expected utility of 0.2025. As the patient will most likely follow the treatment plan only once, they aim to optimise their expected utility and thus prefer distribution $B$.
\end{example}

As this example shows, it is pertinent to consider exactly what the decision maker aims to optimise for: do they optimise for repeated execution of the same policy, or maximising the expected utility from one execution? In the former case, they may well decide based on the expected value of the distribution. In the latter case, however, taking the full distribution of returns into account is key to effective decision support.

\subsection{Distributional Dominance}
To address the limitations of Pareto dominance, we introduce the \emph{distributional dominance} criterion. This criterion states that a distribution dominates another when it is first-order stochastic dominant and at least one of the marginal distributions \emph{strictly} first-order stochastic dominates the related marginal distribution of the second distribution. 

\begin{definition}
\label{def:dist-dom}
A policy $\pi$ distributionally dominates another policy $\pi'$, denoted by $\zpi \du \zpip$, when,
\begin{equation*}
    \zpi \fsde \zpip \land \exists i \in [d]: Z_i^\pi \fsd Z_i^{\pi'}.
\end{equation*}
\end{definition}

One can verify that distributional dominance is equivalent to strict first-order stochastic dominance in the case of random vectors when all variables are independent. In general, however, distributional dominance is a stronger condition than strict first-order stochastic dominance as the condition on the marginal distributions implies strict FSD but is not implied by it. Defining distributional dominance as such enables us to guarantee a strictly greater expected utility for a large class of decision makers and leads to the general solution set discussed in \cref{sec:solution-set}.

For the class of decision makers with utility functions in $\mathcal{U}$, we show that when a given random vector has \emph{strictly} greater expected utility for all utility functions than a second random vector, this implies distributional dominance.

\begin{theorem}
\label{th:u-implies-dd}
Let $\mathbf{X}$ and $\mathbf{Y}$ be d-dimensional random vectors. Then,
\begin{equation*}
    \forall u \in \mathcal{U}: \E u(\mathbf{X}) > \E u(\mathbf{Y}) \implies \mathbf{X} \du \mathbf{Y}.
\end{equation*}
\end{theorem}

\begin{proofsketch}
We first show an additional lemma stating that the condition implies first-order stochastic dominance. Therefore, the proof reduces to showing the condition on the marginals. It suffices to show that if $\mathbf{X}$ does not distributionally dominate $\mathbf{Y}$, it is always possible to construct a utility function for which $\E u(\mathbf{Y})$ is at least as high as $\E u(\mathbf{X})$. 
\end{proofsketch}

In practice, it is impossible to verify whether the expected utility of a given random vector is always strictly greater than that of a second random vector. On the other hand, we will demonstrate that it is computationally feasible to verify distributional dominance (see \cref{sec:computing-dus}). We now show that distributional dominance implies a strictly greater expected utility for a subset of utility functions in $\mathcal{U}$. The condition we impose is referred to as ``multivariate risk-aversion", which means that a decision maker in this class will, when confronted with a choice between two lotteries, always avoid the lottery containing the worst possible outcome \cite{richard1975multivariate}. Below, we present the theorem and proof for bivariate distributions. We note that for FSD this property has been shown to hold for $n$-dimensional random vectors as well \cite{scarsini1988dominance}.

\begin{theorem}
\label{th:dd-implies-u}
Let $\mathbf{X}$ and $\mathbf{Y}$ be two-dimensional random vectors. Then $\forall u \in \mathcal{U}$ with $\frac{\partial^2 u(x_1, x_2)}{\partial x_1 \partial x_2} \leq 0$,
\begin{equation*}
    \mathbf{X} \du \mathbf{Y} \implies \E u(\mathbf{X}) > \E u(\mathbf{Y}) .
\end{equation*} 
\end{theorem}
\begin{proofsketch}
The proof utilises the fact that first-order stochastic dominance implies greater or equal expected utility \cite{hayes2022expected}. We subsequently show that the additional condition on the marginal distributions for distributional dominance implies strictly greater expected utility.
\end{proofsketch}

\section{A General Solution Set}
\label{sec:solution-set}
We adopt distributional dominance to define the distributional undominated set (DUS). The DUS has two important desiderata: it contains the Pareto front, i.e. the optimal set under SER and contains all optimal policies for multivariate risk-averse decision makers under ESR. The deferred proofs for the theoretical results can be found in the supplementary material.

\subsection{Distributional Undominated Set}
As the name suggests, the distributional undominated set contains only those policies which are not pairwise distributionally dominated. We define this formally in \cref{def:du-set}.

\begin{definition}
\label{def:du-set}
The distributional undominated set is the set of all policies that are not distributionally dominated:
\begin{equation}
    \duset = \left\{ \pi \in \Pi \mid \nexists \pi' \in \Pi, \zpip \du \zpi \right\}.
\end{equation}
\end{definition}

From this definition it is clear that all policies which are optimal for multivariate risk-averse decision makers are in the set. To show that the Pareto front is a subset as well, we first introduce \cref{lemma:dd-implies-pd}, stating that distributional dominance implies Pareto dominance.

\begin{lemma}
\label{lemma:dd-implies-pd}
For all policies $\pi, \pi' \in \Pi$,
\begin{equation*}
    \zpi \du \zpip \implies \vpi \pd \vpip.
\end{equation*}
\end{lemma}

\begin{proofsketch}
The proof works by utilising a known link between the expected value of a random variable and its cumulative density function. Then, the conditions for distributional dominance imply that the expected value for each marginal distribution is greater or equal and at least one marginal distribution is strictly greater.
\end{proofsketch}

Leveraging \cref{lemma:dd-implies-pd}, it is a straightforward corollary that the Pareto front is a subset of the DUS. 
\begin{corollary}
\label{co:pf-subset-duset}
For any family of policies $\Pi$, the Pareto front is a subset of the distributional undominated set, i.e., 
\begin{equation*}
    \pf \subseteq \duset.
\end{equation*}
\end{corollary}

We highlight that our dominance results and solution sets are not restricted to MOMDPs but apply to any stochastic multi-objective decision problem with vector-valued outcomes.

\subsection{Computing the DUS}
\label{sec:computing-dus}
To deal with return distributions computationally, we project distributions to multivariate categorical distributions \cite{bellemare2023distributional,hayes2022expected}. This ensures that finite memory is used, and, importantly, that computations can be performed efficiently. Concretely, to verify first-order stochastic dominance, we need only compare a finite number of points as the CDF is a multivariate step function with steps at $\mathbf{v}_1, \mathbf{v}_2, \dotsc, \mathbf{v}_n$. Formally, for the categorical distribution $\mathbf{X}$ the cumulative distribution at $\mathbf{x}$ is computed as follows,

\begin{equation}
    F_{\mathbf{X}}(\mathbf{x}) = \sum_{\mathbf{v}_i \pdef \mathbf{x}} p(\mathbf{v}_i).
\end{equation}

Additionally, discrete distributions enable straightforward computation of marginal distributions, thus having all ingredients to check distributional dominance (see \cref{def:dist-dom}). Then, starting from a given set of policies, the DUS can be computed using a modified version of the Pareto Prune (PPrune) algorithm \cite{roijers2017multiobjective} that checks for distributional dominance rather than Pareto dominance. We refer to the resulting pruning algorithm as \emph{DPrune}.

\section{A Solution Set for ESR}
As the DUS is a superset of the Pareto front and further contains optimal policies under ESR, we can intuitively assume that it might grow very large in size, thereby complicating its practical use in decision support systems. When considering SER, it is possible to reduce the set to the Pareto front by utilising existing pruning operators \cite{roijers2017multiobjective}. We contribute a similar approach for ESR and present both the resulting solution set as well as a pruning algorithm for this purpose.

\subsection{Convex Mixture of Distributions}
For univariate distributions, it has been shown that a mixture distribution can be constructed that first-order stochastic dominates another distribution if and only if for any decision maker there exists a distribution in the mixture which is preferred over the dominated distribution \cite{fishburn1974convex,bawa1985determination}. Mixture dominance has also been considered for multivariate distributions \cite{denuit2013multivariate}. 

Here, we show that convex distributional dominance implies greater expected utility for multivariate risk-averse decision makers when considering bivariate distributions. 

\begin{theorem}
\label{th:mixture-dom-implies-u}
Let $\{\mathbf{X}_1, \dotsc, \mathbf{X}_n\}$ and $\{\mathbf{Y}_1, \dotsc, \mathbf{Y}_n\}$ be sets of two-dimensional random vectors. Then,
\begin{equation*}
    \exists \lambda \in \Delta^n: \sum_{i=1}^n \lambda_i \mathbf{X}_i \du \sum_{i=1}^n \lambda_i \mathbf{Y}_i,
\end{equation*}
implies that $\forall u \in \mathcal{U}$ with $\frac{\partial^2 u(x_1, x_2)}{\partial x_1 \partial x_2} \leq 0$,
\begin{equation*}
    \exists i \in [n]: \E u(\mathbf{X}_i) > \E u(\mathbf{Y}_i).
\end{equation*}
\end{theorem}
\begin{proofsketch}
The proof follows from \cref{th:dd-implies-u} and linearity of expectation.
\end{proofsketch}

Observe that in the special case where all random vectors $\mathbf{Y}_i$ are equal, mixture dominance of $\mathbf{Y}$ implies that all decision makers will prefer a random vector $\mathbf{X}_i$ over $\mathbf{Y}$. 

\subsection{Convex Distributional Undominated Set}
We define a final solution set, called the convex distributional undominated set (CDUS), that contains only those policies which are undominated by a mixture of distributions. \Cref{th:mixture-dom-implies-u} guarantees that for all decision makers in the class there is an optimal policy contained in the set. We define the CDUS formally below. It follows from this definition that the CDUS is a subset of the DUS.

\begin{definition}
\label{def:convex-du-set}
The CDUS is the set of all policies that are not distributionally dominated by a convex mixture:
\begin{equation*}
    \cduset = \left\{ \pi \in \Pi \mid \nexists \lambda \in \Delta^{|\Pi|}: \sum_{i=1}^{|\Pi|} \lambda_i \mathbf{Z}^{\pi_i} \du \zpi\right\}.
\end{equation*}
\end{definition}

Given the myriad of solution sets in multi-objective decision making, it is useful to define a complete taxonomy between them. From \cref{co:pf-subset-duset}, we know that the Pareto front is a subset of the DUS. Additionally, it follows from \cref{def:convex-du-set} that the CDUS is also a subset of the DUS. Earlier work has shown that the convex hull is a subset of the Pareto front \cite{roijers2017multiobjective} and we show that this is also true for the CDUS.

\begin{corollary}
\label{co:ch-subset-cesr}
For any family of policies $\Pi$,
\begin{equation*}
    \text{CH}(\Pi) \subseteq \cduset.
\end{equation*}
\end{corollary}

The final missing piece of the puzzle is the relation between the CDUS and Pareto front. However, here one can find counterexamples which disprove that the CDUS is either a subset or superset of the Pareto front. The landscape of solution sets for multi-objective decision making can then be summarised as shown in \cref{fig:solution-sets}.

\begin{figure}
    \centering
    \includegraphics{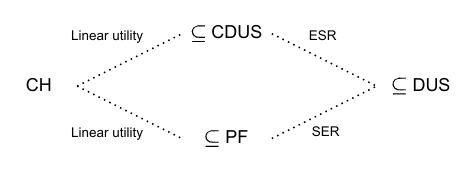}
    \caption{A taxonomy of solution sets in multi-objective decision making.}
    \label{fig:solution-sets}
\end{figure}

\subsection{Pruning to the CDUS}
To prune a set of distributions to its CDUS, we must check for each distribution whether it is dominated by a mixture of the other distributions. Fortunately, this verification is feasible by restating the problem using linear programming. Concretely, we extend an algorithm that checks whether a univariate distribution is convex first-order stochastic dominated to our setting \cite{bawa1985determination}. We show the resulting linear program \emph{CDPrune} in \cref{alg:cdd}.

For notational simplicity, we define the size of the set of distributions allowed in the mixture as $n$. Then the linear program takes in total $n+1$ distributions as input, where the final distribution is the distribution to check. As these distributions are discrete, the CDFs are multivariate step functions that step at a finite number of points. Let $D_i$ be the set of points at which the CDF of distribution $i$ steps. Then $D = \bigcup_{i=1}^{n+1} D_i$ is the union of all such points. We denote $h = |D|$.

\begin{algorithm}[t]
\caption{CDPrune}
\label{alg:cdd}
    \begin{algorithmic}[]
    \Require A set of return distributions $\mathcal{Z}$ allowed in the mixture and a return distribution $\mathbf{Z}$ to check
    \Ensure Whether the distribution is convex dominated
    \State \begin{align}
    & \text{Maximise } \delta = \sum_{i=1}^n \sum_{k=1}^d l_{i,k} \label{eq:maximisation}\\
    & \text{Subject to:} \nonumber\\
    & \quad \sum_{i=1}^n \lambda_i F_{\mathcal{Z}_i}(\mathbf{v}_j) + s_{j} = F_{\mathbf{Z}}(\mathbf{v}_j) \quad j = 1, \dotsc, h \label{eq:joint-constraint}\\
    & \quad \sum_{i=1}^n \lambda_i F_{\mathcal{Z}_{i,k}}(v_{j,k}) + l_{j,k} = F_{Z_k}(v_{j,k}) \nonumber \\
    & \quad \qquad j = 1, \dotsc, h \quad k = 1, \dotsc, d \label{eq:marginal-constraint}\\
    & \quad \sum_{i=1}^n \lambda_i = 1 \nonumber \\
    & \quad \lambda_i \geq 0 \quad i = 1, \dotsc, n \nonumber\\
    & \quad s_j \geq 0 \quad j = 1, \dotsc, h \quad \text{where } s_j \text{ is a slack variable} \label{eq:slack-variables}
    \end{align}
    \State \Return $\textproc{True}$ \bf{if} $\delta > 0$ \bf{else} $\textproc{False}$
    \end{algorithmic}
\end{algorithm}

The linear program maximises $\delta$, which is the sum of slack variables that make up the difference between the CDFs of the marginal mixture distributions and the marginals of the distribution to check (\cref{eq:maximisation}). If this procedure leads to a $\delta$ greater than zero, this implies that the conditions for distributional dominance are met and the distribution is dominated by the mixture. Note that we may omit an additional constraint on the $l$ slack variables to be greater or equal to zero, as this is implied by the constraint on the $s$ slack variables (\cref{eq:slack-variables}).


When no exact formulation of the joint CDFs is available, we propose an alternative linear program that operates solely on the marginal distributions. In this case, it is necessary to change the first constraint in \cref{eq:joint-constraint} to 
\begin{equation}
\sum_{i=1}^n \lambda_i \prod_{k=1}^d F_{\mathcal{Z}_{i,k}}(\mathbf{v}_{j,k}) + s_{j} = \prod_{k=1}^d F_{Z_k}(v_{j,k}),
\end{equation}
while the second constraint in \cref{eq:marginal-constraint} is removed altogether. By maximising the sum of $s$ slack variables, the resulting linear program essentially checks for strict first-order stochastic dominance between random vectors with independent variables. One can verify that this implies distributional dominance for independent variables and otherwise may serve as an approximation.

\section{Computing the Solution Sets}
Our final contribution relates theory to practice by designing an algorithm able to learn the DUS in a given MOMDP. We evaluate this algorithm on different sizes of MOMDPs and compare the resulting sizes of the sets when pruned down to the subsets covered in the taxonomy in \cref{fig:solution-sets}. All code is available at \url{https://github.com/wilrop/distributional-dominance}.

\subsection{Distributional Multi-Objective Q-Learning}
Pareto Q-learning (PQL) is a classical algorithm used in multi-objective reinforcement learning to learn the Pareto front \cite{vanmoffaert2014multiobjective}. We find that the general framework of PQL lends itself nicely to learning the DUS. Our algorithm, DIstributional Multi-Objective Q-learning (DIMOQ) is shown in \cref{alg:dimoq}. 


\begin{algorithm}[]
\caption{DIMOQ}
\label{alg:dimoq}
    \begin{algorithmic}[1]
    \Require The state space $\mathcal{S}$, actions space $ \mathcal{A}$ and discount factor $\gamma$
    \Ensure The DUS
    \State Initialise all $Q(s, a)$ as empty sets
    \State Initialise all $\mathbf{R}(s, a, s')$ as Dirac delta distributions 
    \State Estimate $T: \mathcal{S} \times \mathcal{A} \times \mathcal{S} \to [0, 1]$ from random walks
    \For{each episode}
    \State Initialise state $s$
    \Repeat
    \State Take an action $a \sim \pi(a|s)$
    \State Observe the next state $s' \in \mathcal{S}$ and reward $\mathbf{r} \in \mathbb{R}^d$
    \State $ND(s, a, s') \gets \textproc{DPrune} \left(\bigcup_{a' \in \mathcal{A}} Q(s', a')\right)$
    \State Update the reward distribution $\mathbf{R}(s, a, s')$ with $\mathbf{r}$
    \State $s \gets s'$
    \Until{$s$ is terminal}
    \EndFor
    \State \Return $\textproc{DPrune}\left(\bigcup_{a \in \mathcal{A}}Q\left(0, a\right)\right)$
    \end{algorithmic}
\end{algorithm}

The algorithm first initialises the Q-sets containing undominated distributions to empty sets and reward distributions to Dirac delta distributions at zero. During training, the agent follows an $\epsilon$-greedy policy and learns the immediate reward distributions $\mathbf{R}(s, a, s')$ separate from the expected future reward distributions $ND(s, a, s')$. Learning the immediate reward distribution is done by recording the empirical distribution, while learning the future reward distribution is done using a modified version of the Q-update rule employed for PQL (see \cref{eq:q-update}). Note, however, that for DIMOQ the pruning operator for the distributions in the next state is \emph{DPrune} rather than \emph{PPrune}.

\subsubsection{Dealing With Stochasticity}
\label{sec:stochasticity}
The Q-learning update in PQL is described for deterministic environments. As we deal with fundamentally stochastic environments, we propose an alternative formulation in \cref{eq:q-update}. 

\begin{equation}
\label{eq:q-update}
Q(s,a) \gets \bigoplus_{s'}T(s'|s, a)\left[\mathbf{R}(s, a, s') + \gamma ND(s, a, s')\right]
\end{equation}

First, the term $\left[\mathbf{R}(s, a, s') + \gamma ND(s, a, s')\right]$ constructs a set of expected return distributions when the state-action pair leads to $s'$. Next, the $\bigoplus_{s'}T(s'|s, a)$ constructs mixture policies over all next states $s'$ where each distribution is weighted according to its transition probability $T(s'|s, a)$. 

In a learning setting, the transition probabilities are not assumed to be given. As such, we perform a number of random walks before training to estimate these probabilities. During learning, we do not update the transition function anymore, to avoid creating unnecessary distributions which will never be observed again due to drift in the probabilities. 

\subsubsection{Action Selection}
The second adaptation necessary to learn the DUS rather than the Pareto front is the action scoring and selection mechanism. Even for PQL, this is complicated as it is not obvious what metric to use to determine the quality of a set of Q-values. Several set evaluation mechanisms have been proposed for this, such as for example the hypervolume metric \cite{guerreiro2022hypervolume} or using a Chebyshev scalarisation function \cite{vanmoffaert2013scalarized}. We note that these approaches can be extended to DIMOQ as well by computing the expected value of the distribution first and then continuing with one of the aforementioned scoring metrics. 

In addition to the classical scoring methods, we propose using a linear utility function as a baseline and scoring a set of distributions by its mean expected utility. As linear scalarisation can be done efficiently, this results in a performant scoring method. An additional advantage of this approach is that when more information about the shape of the utility function is known, the linear utility baseline can be substituted with a better approximation.

\subsubsection{Limiting Set Size}
Due to stochasticity in the environment and because the Q-update rule in \cref{eq:q-update} performs all possible combinations, the Q-sets in the algorithm are quick to explode in size. To constrain the size of the sets, we propose two mechanisms. 

First, we limit the precision of the distributions that are learned. This approach was demonstrated to be successful in multi-objective dynamic programming as well \cite{mandow2022multiobjective}. Second, we set a fixed limit on the set size. Whenever this limit is crossed, we perform agglomerative clustering where the number of clusters equals the maximum set size. As input for the clustering, we compute the pairwise distances between all distributions. In experiments, we compute the Jensen-Shannon distance between the flattened distributions. Alternatively, one could use the cost of optimal transport between pairs of distributions.

\subsection{Empirical Results}
We evaluate DIMOQ (\cref{alg:dimoq}) and CDPrune (\cref{alg:cdd}) on randomly generated MOMDPs of different sizes shown in \cref{tab:momdp-config}. For each size category, we repeat the experiment with seeds one through five and perform $50,000$ random walks to estimate $T$ followed by $2,000$ training episodes. All experiments considered two objectives, used a discount factor of $1$ and limited the precision of distributions to three decimals. Finally, the experiments were run on a single core of an Intel Xeon Gold 6148 processor, with a maximum RAM requirement of 2GB. 

\begin{table}[]
\centering
\resizebox{\columnwidth}{!}{%
\begin{tabular}{cccccc}
\toprule
\textbf{Name} & \textbf{States} & \textbf{Actions} & \textbf{Next states} & \textbf{Timesteps} & \textbf{Set limit} \\ \midrule
\textit{Small}                     & 5                          & 2                           & $[1, 2]$                          & 3                             & 10                            \\
\textit{Medium}                    & 10                         & 3                           & $[1, 2]$                          & 5                             & 15                            \\
\textit{Large}                     & 15                          & 4                           & $[1, 2]$                          & 7                             & 20    \\
\bottomrule
\end{tabular}%
}
\caption{Configuration of the generated MOMDPs. Timesteps refer to the maximum time horizon after which the episode is terminated.}
\label{tab:momdp-config}
\end{table}

We observe that the runtimes for DIMOQ shown in \cref{tab:dimoq} are heavily influenced by the size of the MOMDP. Additionally, there is a large variance in runtime across different seeds. We find that these differences cannot solely be attributed to having a more complex transition function, but are most likely due to the interplay between the transition function and the reward function. Specifically, if transitions result in a large number of undominated returns each iteration needs to perform a large number of combinations. It is clear however that scaling becomes an issue for DIMOQ when going to larger action and state spaces. As such, we plan to investigate the use of function approximation to further extend DIMOQ to larger MOMDPs. Additionally, we note that MOMDPs modelled after real-world scenarios will likely contain more structure and are thus interesting to study for future work.

\begin{table}[]
\centering
\resizebox{\columnwidth}{!}{%
\begin{tabular}{ccccc}
\toprule
\textbf{Name} & \textbf{Mean} & \textbf{SD} & \textbf{Min} & \textbf{Max} \\ \midrule
\textit{Small}                     & 00:01:21                 & 00:00:25                & 00:00:58                & 00:02:01                \\
\textit{Medium}                    & 01:49:11                 & 00:47:07                & 00:17:41                & 02:31:18                \\
\textit{Large}                     & 17:01:25                 & 06:02:35                & 09:46:06                & 27:55:55       \\
\bottomrule
\end{tabular}%
}
\caption{Runtime for DIMOQ on randomly generated MOMDPs.}
\label{tab:dimoq}
\end{table}

In \cref{tab:pruning} we show the average size of the DUS, as well as what percentage of the DUS belongs to the CDUS, Pareto front and convex hull on average. We observe a similar pattern, namely that larger MOMDPs lead to larger solution sets. Interestingly though, larger MOMDPs also allow for a greater percentage of policies to be pruned for the smaller solution sets, which is beneficial for their use in decision support.

\begin{table}[]
\centering
\resizebox{\columnwidth}{!}{%
\begin{tabular}{ccccc}
\toprule
\textbf{Name}            & \textbf{DUS}                & \textbf{CDUS}                & \textbf{PF}                  & \textbf{CH}                  \\ \midrule
\textit{Small}  & $13.0 \pm 10.73$   & $95.71\% \pm 8.57$  & $39.88\% \pm 16.45$ & $36.07\% \pm 20.12$ \\
\textit{Medium} & $372.2 \pm 211.88$ & $61.27\% \pm 12.16$ & $6.28\% \pm 6.70$   & 2$.87\% \pm 3.48$   \\
\textit{Large}  & $639.0 \pm 221.71$ & $53.00\% \pm 5.68$  & $3.43\% \pm 2.01$   & $1.33\% \pm 0.82$  \\
\bottomrule
\end{tabular}%
}
\caption{The relative sizes of the pruned subsets.}
\label{tab:pruning}
\end{table}

We highlight that although the CDUS is often substantially smaller than the DUS, the Pareto front and convex hull are much smaller than either. Intuitively, this is because when both objectives are to be maximised, Pareto optimal policies can only occur on the upper right hand region of the objective space, while policies in the DUS and CDUS may still exist in the Pareto dominated part of the space. However, recall from \cref{exmp:utility} that these policies may still be optimal under ESR. We visualise this in \cref{fig:visualisation} where the expected values for the final distributions from one representative experiment are plotted. 

\begin{figure}
    \centering
    \includegraphics[width=0.85\columnwidth]{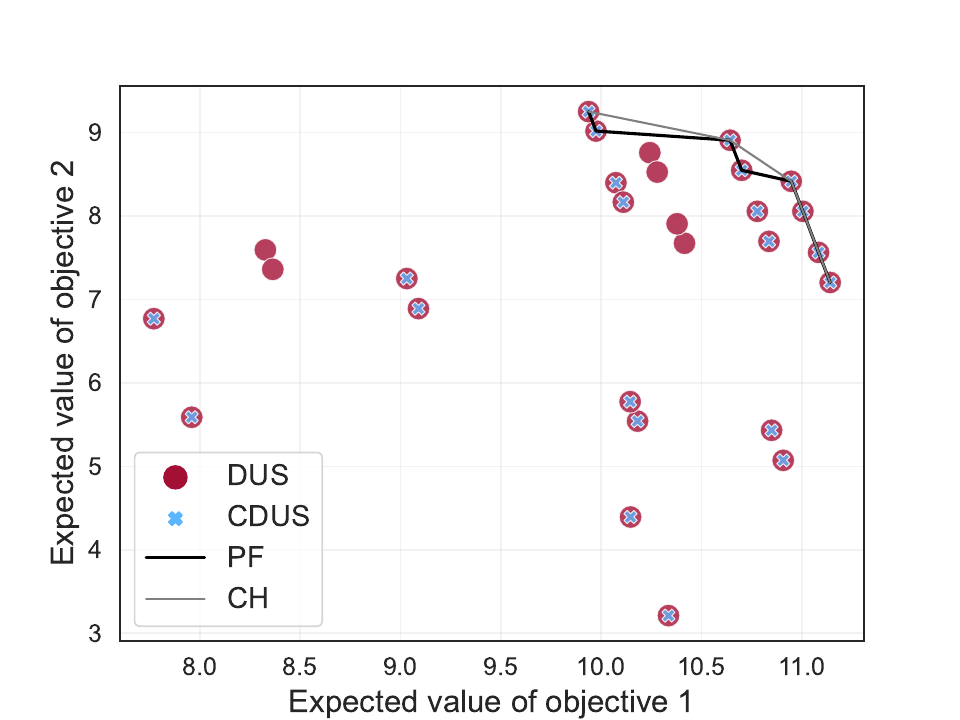}
    \caption{The resulting solution sets for a sample experiment. Policies in the dominated part of the objective space may still be optimal for certain decision makers and can thus not be excluded a priori.}
    \label{fig:visualisation}
\end{figure}

Finally, we remark that while the CDUS cannot be guaranteed to be a superset of the Pareto front in general, in all experiments this was in fact the case. This is also apparent from the results in \cref{fig:visualisation}. An interesting direction for future work is to specify the exact conditions under which this relation is guaranteed to hold.

\section{Related Work}
\label{sec:related-work}
Stochastic dominance has long been employed in areas of finance and economics \cite{levy2016stochastic} and has more recently also found use in solving decision making problems through reinforcement learning (RL). In single-objective settings, Epshteyn and DeJong~\shortcite{epshteyn2006qualitative} employ stochastic dominance to learn optimal policies in MDPs with incomplete specifications. Martin et al.~\shortcite{martin2020stochastically} define a risk-aware distributional algorithm that utilises stochastic dominance at decision time to determine the best action. Techniques from stochastic dominance have also been used to analyse the theoretical properties of distributional RL \cite{rowland2018analysis}.

The distributional approach in general has become an active area of research for both single-objective and multi-objective settings. For a thorough overview of techniques in single-objective settings, we refer to a recent textbook on the matter \cite{bellemare2023distributional}. In multi-objective settings, Hayes et al.~\shortcite{hayes2021dmcts} and Reymond et al.~\shortcite{reymond2023actorcritic} define single-policy multi-objective RL algorithms that can learn policies for nonlinear utility functions under the ESR criterion. Furthermore, Hayes et al.~\shortcite{hayes2022decision} outline a multi-policy multi-objective distributional value iteration algorithm that computes a set of policies for the ESR criterion, known as the ESR set. The ESR set is the first solution set for use in multi-objective sequential decision making under the ESR criterion and leverages strict first-order stochastic dominance to determine whether a policy is included in the set. This set was shown to contain all optimal policies for multivariate risk-averse decision makers, but implicitly assumes all variables in the random vector to be independent \cite{hayes2022expected}.

\section{Conclusion}
\label{sec:conclusion}
We investigate multi-objective decision making and find that existing solution sets frequently fall short in specific use cases. To resolve this, we first propose the distributional undominated set. We show that this set contains both the Pareto front as well as all optimal policies for multivariate risk-averse decision makers optimising their expected utility. We subsequently present the convex distributional undominated set, which aims to target the expected utility setting in particular. From this, we determine a taxonomy of existing solution sets in multi-objective decision making. 

To facilitate the application of these concepts, we present computational approaches for learning the distributional undominated set and pruning operators to reduce the set to the convex distributional undominated set. Through experiments, we demonstrate the feasibility and effectiveness of these methods. As such, this work offers a promising approach to decision support in real-world problems.

\section*{Acknowledgments}
WR is supported by the Research Foundation – Flanders (FWO), grant numbers 1197622N. CH is funded by the University of Galway Hardiman Scholarship. This research was supported by funding from the Flemish Government under the ``Onderzoeksprogramma Artifici\"{e}le Intelligentie (AI) Vlaanderen'' program.

\bibliographystyle{named}
\bibliography{bibliography}
\clearpage
\appendix

\section{Proofs of Theoretical Results}
In this section, we present the deferred proofs from the main text.

\subsection{Distributional Dominance}
We introduce several theoretical results concerning distributional dominance. In particular, we show that a greater expected utility for all strictly increasing utility functions implies distributional dominance. To prove this, we present two necessary lemmas. First, \cref{lemma:u-implies-fsd} is a straightforward generalisation to an earlier result by Fishburn~\shortcite{fishburn1974convex}.

\begin{lemma}
\label{lemma:u-implies-fsd}
Let $\mathbf{X}$ and $\mathbf{Y}$ be d-dimensional random vectors. Then,
\begin{equation*}
    \forall u \in \mathcal{U}: \E u(\mathbf{X}) \geq \E u(\mathbf{Y}) \implies \mathbf{X} \fsde \mathbf{Y}.
\end{equation*}
\end{lemma}
\begin{proof}
We show this by contradiction. Let $\mathbf{v} \in \mathbb{R}^d$ such that $F_\mathbf{X}(\mathbf{v}) > F_\mathbf{Y}(\mathbf{v})$. Define $u$ to be a smooth approximation to the multivariate step function with $u(\mathbf{z}) = 0$ when $\mathbf{z} \preceq_\text{p} \mathbf{v}$ and $1$ otherwise. It is clear that $u \in \mathcal{U}$. Then,
\begin{align}
\E u(\mathbf{X}) - \E u(\mathbf{Y}) & = (1 - F_\mathbf{X}(\mathbf{v})) - (1 - F_\mathbf{Y}(\mathbf{v})) \\
& = F_\mathbf{Y}(\mathbf{v}) - F_\mathbf{X}(\mathbf{v}) \\
& < 0. \qedhere
\end{align}
\end{proof}

The second lemma relates first-order stochastic dominance on the joint distribution to the same restriction on all of the marginal distributions. 
\begin{lemma}
\label{lemma:cdf-marginals-fsd}
Let $\mathbf{X}$ and $\mathbf{Y}$ be d-dimensional random vectors. Then,
\begin{equation*}
    \mathbf{X} \fsde \mathbf{Y} \implies \forall i \in [d]: X_i \fsde Y_i.
\end{equation*}
\end{lemma}
\begin{proof}
First, note that for any random vector $\mathbf{X}$,

\begin{equation}
    F_{X_i}(x) = \lim_{\mathbf{x}_{-i} \to \infty} F_{\mathbf{X}_{-i}, X_i}(\mathbf{x}_{-i}, x_i).
\end{equation}

Thus, when $\mathbf{X} \fsde \mathbf{Y}$, then $\forall \mathbf{v} \in \mathbb{R}^d$:
\begin{align}
    \implies & F_\mathbf{X}(\mathbf{v}) \leq  F_\mathbf{Y}(\mathbf{v}) \\
    \implies & \lim_{\mathbf{v}_{-i} \to \infty} F_{\mathbf{X}_{-i}, X_i}(\mathbf{v}_{-i}, v_i) \\
    & \leq \lim_{\mathbf{v}_{-i} \to \infty} F_{\mathbf{Y}_{-i}, Y_i}(\mathbf{v}_{-i}, v_i) \\
    \implies & \forall i \in [d]: F_{X_i}(v_i) \leq F_{Y_i}(v_i) \\
    \implies & \forall i \in [d]: X_i \fsde Y_i. \qedhere
\end{align}
\end{proof}

Using \cref{lemma:u-implies-fsd,lemma:cdf-marginals-fsd} we can show the required result. We repeat Theorem 3.1 first and present the proof immediately after.
\begin{theorem31}
Let $\mathbf{X}$ and $\mathbf{Y}$ be d-dimensional random vectors. Then,
\begin{equation*}
    \forall u \in \mathcal{U}: \E u(\mathbf{X}) > \E u(\mathbf{Y}) \implies \mathbf{X} \du \mathbf{Y}.
\end{equation*}
\end{theorem31}

\begin{proof}
We know from \cref{lemma:u-implies-fsd} that,
\begin{equation}
\forall u \in \mathcal{U}: \E u(\mathbf{X}) > \E u(\mathbf{Y}) \implies \mathbf{X} \fsde \mathbf{Y}.
\end{equation}
Therefore, we need only show that $\exists i \in [d]: X_i \fsd Y_i$. 

We define a utility function $u$ that is a sum of univariate strictly monotonically increasing utility functions where each term only takes one variable from the vector into account, i.e.,
\begin{equation}
u(\mathbf{x}) = u_1(x_1) + \dotsb + u_d(x_d).
\end{equation}

Then,

\begin{align}
& \E u(\mathbf{X}) > \E u(\mathbf{Y}) \\
\implies & \int_{\mathbf{X}}u(\mathbf{x}) f_{\mathbf{X}}(\mathbf{x}) d\mathbf{x} > \int_{\mathbf{Y}}u(\mathbf{y}) f_{\mathbf{Y}}(\mathbf{y}) d\mathbf{y} \\
\implies & \sum_{i=1}^d \int_{-\infty}^{+\infty}\int_{\mathbf{X_{-i}}}u_i(x_i) f_{\mathbf{X}}(\mathbf{x_{-i}}, x_i)d\mathbf{x_{-i}}dx_{i} \\
& > \sum_{i=1}^d \int_{-\infty}^{+\infty}\int_{\mathbf{Y_{-i}}}u_i(y_i)f_{\mathbf{Y}}(\mathbf{y_{-i}}, y_i)d\mathbf{y_{-i}}dy_{i}\\
\implies & \sum_{i=1}^d \int_{-\infty}^{+\infty} u_i(x_i)  \int_{\mathbf{X_{-i}}}f_{\mathbf{X}}(\mathbf{x_{-i}}, x_i)d\mathbf{x_{-i}}dx_{i} \\
& > \sum_{i=1}^d \int_{-\infty}^{+\infty} u_i(y_i) \int_{\mathbf{Y_{-i}}} f_{\mathbf{Y}}(\mathbf{y_{-i}}, y_i)d\mathbf{y_{-i}}dy_{i}\\
\implies & \sum_{i=1}^d \int_{-\infty}^{+\infty}u_i(x_i) f_{X_i}(x_i)dx_{i} \\
& > \sum_{i=1}^d \int_{-\infty}^{+\infty}u_i(y_i) f_{Y_i}(y_i)dy_{i}
\end{align}

Observe that \cref{lemma:cdf-marginals-fsd} guarantees that,
\begin{equation}
    \mathbf{X} \fsde \mathbf{Y} \implies \forall i \in [d]: X_i \fsde Y_i.
\end{equation}
As such, the final implication is only true when $\exists i \in [d]: X_i \fsd Y_i$.
\end{proof}

The second deferred proof from Section 3 showed that distributional dominance implies strictly greater expected utility for the class of multivariate risk-averse decision makers.
\begin{theorem32}
Let $\mathbf{X}$ and $\mathbf{Y}$ be two-dimensional random vectors. Then $\forall u \in \mathcal{U}$ with $\frac{\partial^2 u(x_1, x_2)}{\partial x_1 \partial x_2} \leq 0$,
\begin{equation*}
     \mathbf{X} \du \mathbf{Y} \implies \E u(\mathbf{X}) > \E u(\mathbf{Y}) .
\end{equation*} 
\end{theorem32}

\begin{proof}
By definition, $\mathbf{X} \du \mathbf{Y} \implies \mathbf{X} \fsde \mathbf{Y}$. For this condition, Hayes et al.~\shortcite{hayes2022expected} show that,
\begin{equation}
\label{eq:hayes-result}
\E u(\mathbf{X}) - \E u(\mathbf{Y}) \geq - \int_{-\infty}^{+\infty} \lim_{t \to + \infty} \frac{\partial u(t, z)}{\partial z} \Delta_F(t,z) dz.
\end{equation}
Where $\Delta_F(t,z) = F_{\mathbf{X}}(t,z) - F_{\mathbf{Y}}(t,z)$. Without loss of generality, let us assume that $X_2 \fsd Y_2$, i.e. $\exists z \in \mathbb{R}: F_{X_2}(z) < F_{Y_2}(z)$. Then,
\begin{align}
& \E u(\mathbf{X}) - \E u(\mathbf{Y}) \\
& \geq - \int_{-\infty}^{+\infty} \lim_{t \to + \infty} \frac{\partial u(t, z)}{\partial z} \Delta_F(t,z) dz \\
& = - \int_{-\infty}^{+\infty} \left(\lim_{t \to + \infty} \frac{\partial u(t, z)}{\partial z}\right) \Delta_{F_{2}}(z) dz \\
& > 0. \qedhere
\end{align}
\end{proof}

\subsection{Distributional Undominated Set}
Based on the distributional dominance criterion, we define the Distributional Undominated Set (DUS) in Definition 4.1. To show that this set is a superset of the Pareto front, we restate Lemma 4.1 below and subsequently present the missing proof.

\begin{lemma41}
For all policies $\pi, \pi' \in \Pi$,
\begin{equation*}
    \zpi \du \zpip \implies \vpi \pd \vpip.
\end{equation*}
\end{lemma41}
\begin{proof}
Let $\zpi$ and $\zpip$ be two return distributions and let $\zpi \du \zpip$. We can expand the terms as follows,
\begin{equation}
    \zpi = \begin{bmatrix}
       Z_1^{\pi} \\[0.3em]
       \vdots \\[0.3em]
       Z_d^{\pi}
     \end{bmatrix} \quad \text{and} \quad
     \zpip = \begin{bmatrix}
       Z_1^{\pi'} \\[0.3em]
       \vdots \\[0.3em]
       Z_d^{\pi'}
     \end{bmatrix},
\end{equation}
with $d$ the number of objectives. We know that because $\zpi \du \zpip$, 
\begin{equation}
\forall i \in [d]: Z_i^\pi \fsde Z_i^{\pi'} \land \exists i \in [d]: Z_i^\pi \fsd Z_i^{\pi'}.
\end{equation}
The following can then be shown to hold \cite{hayes2022expected}:
\begin{equation}
\forall i \in [d]: Z_i^\pi \fsde Z_i^{\pi'} \implies \forall i: \E[Z_i^\pi] \geq \E[Z_i^{\pi'}].
\end{equation} 
This ensures that $\vpi \pde \vpip$. Let us denote the $i$ for which strict first-order stochastic dominance holds as $j$. For strict Pareto dominance, observe that
\begin{align}
    \E\left[Z_j^\pi\right] & = \int_0^{+\infty}(1 - F_{Z_j^\pi}(v))dv \\
    \E\left[Z_j^{\pi'}\right] & = \int_0^{+\infty}(1 - F_{Z_j^{\pi'}}(v))dv.
\end{align}
As $F_{Z_j^\pi}(v) \leq F_{Z_j^{\pi'}}(v)$ for all $v$ and strictly less for some $v$, we can say that,
\begin{align}
    & \int_0^{+\infty}(1 - F_{Z_j^\pi}(v))dv > \int_0^{+\infty}(1 - F_{Z_j^{\pi'}}(v))dv \\
    & \implies \E\left[Z_j^\pi\right] > \E\left[Z_j^{\pi'}\right]. \qedhere
\end{align}
\end{proof}

It is then a simple corollary that the Pareto front is a subset of the DUS.

\begin{corollary411}
For any family of policies $\Pi$ the Pareto front is a subset of the distributional dominance set, i.e., 
\begin{equation*}
    \pf \subseteq \duset.
\end{equation*}
\end{corollary411}
\begin{proof}
Assume that there exists a $\pi \in \pf$ such that $\pi \notin \duset$. As $\pi \notin \duset$, we know that,
\begin{equation}
\exists \pi' \in \Pi, \zpip \du \zpi.
\end{equation}
Lemma 4.1 implies then that $\vpip \pd \vpi$. As $\pi$ is Pareto dominated by $\pi'$, $\pi \notin \pf$, leading to a contradiction.
\end{proof}

\subsection{Convex Distributional Undominated Set}
Our final theoretical contributions introduce the Convex Distributional Undominated Set (CDUS). We first extend a well-known result from univariate first-order stochastic dominance to our setting \cite{fishburn1974convex}.
\begin{theorem51}
Let $\{\mathbf{X}_1, \dotsc, \mathbf{X}_n\}$ and $\{\mathbf{Y}_1, \dotsc, \mathbf{Y}_n\}$ be sets of two-dimensional random vectors. Then,
\begin{equation*}
    \exists \lambda \in \Delta^n: \sum_{i=1}^n \lambda_i \mathbf{X}_i \du \sum_{i=1}^n \lambda_i \mathbf{Y}_i,
\end{equation*}
implies that $\forall u \in \mathcal{U}$ with $\frac{\partial^2 u(x_1, x_2)}{\partial x_1 \partial x_2} \leq 0$,
\begin{equation*}
    \exists i \in [n]: \E u(\mathbf{X}_i) > \E u(\mathbf{Y}_i).
\end{equation*}
\end{theorem51}
\begin{proof}
Assume that given $\{\lambda_1, \dotsc, \lambda_n\}$, the condition holds for $\{\mathbf{X}_1, \dotsc, \mathbf{X}_n\}$ and $\{\mathbf{Y}_1, \dotsc, \mathbf{Y}_n\}$. By Theorem 3.2 and linearity of expectation, we can state that $\forall u \in \mathcal{U}$ with $\frac{\partial^2 u(x_1, x_2)}{\partial x_1 \partial x_2} \leq 0$,
\begin{align}
& \E u\left(\sum_{i=1}^n \lambda_i \mathbf{X}_i \right) > \E u \left( \sum_{i=1}^n \lambda_i \mathbf{Y}_i \right) \\
\implies & \sum_{i=1}^n \lambda_i \E u\left(\mathbf{X}_i \right) > \sum_{i=1}^n \lambda_i \E u\left(\mathbf{Y}_i \right) \\
\implies & \exists i \in [n]: \E u(\mathbf{X}_i) > \E u(\mathbf{Y}_i). \qedhere
\end{align}
\end{proof}

Finally, we show the simple corollary that the convex hull is a subset of the convex distributional undominated set.
\begin{corollary511}
For any family of policies $\Pi$,
\begin{equation*}
    \text{CH}(\Pi) \subseteq \cduset.
\end{equation*}
\end{corollary511}
\begin{proof}
Assume that there exists a policy $\pi$ which is in the convex hull but not in the convex distributional dominance set. This implies there exists a set of weights over policies in the CDUS which dominates $\pi$. Then,
\begin{align}
    & \sum_{i=1}^{|\Pi|} \lambda_i \mathbf{Z}^{\pi_i} \du \zpi \\
    \implies & \E \left[\sum_{i=1}^{|\Pi|} \lambda_i \mathbf{Z}^{\pi_i} \right] \pd \E \left[\zpi\right] \\
    \implies & \sum_{i=1}^{|\Pi|} \lambda_i \E \left[\mathbf{Z}^{\pi_i} \right] \pd \E \left[\zpi\right] \\
    \implies & \sum_{i=1}^{|\Pi|} \lambda_i \mathbf{V}^{\pi_i} \pd \vpi.
\end{align}
As such, $\pi$ would also not be in the convex hull. This leads to a contradiction, thus concluding the proof.
\end{proof}

\section{Additional Results}
Throughout this work, we intended to prove results that were as general as possible. In many cases, however, it was necessary to introduce strong conditions to arrive at a result. In this section, we present simple propositions that demonstrate why weaker conditions may fail. To the best of our knowledge, these results have not been published previously and we hope that future researchers can use them to avoid exploring impossible results.

\subsection{First-Order Stochastic Dominance}
First, we show that strict first-order stochastic dominance does not guarantee that one of the marginals strictly first-order stochastic dominates their related marginal. This motivates the need to define distributional dominance as a separate dominance criterion for multivariate distributions.

\begin{proposition}
\label{prop:cdf-marginals-fsd}
Let $\mathbf{X}$ and $\mathbf{Y}$ be d-dimensional random vectors. Then,
\begin{equation*}
\mathbf{X} \fsd \mathbf{Y} \centernot\implies \exists i \in [d]: X_i \fsd Y_i.
\end{equation*}
\end{proposition}
\begin{proof}
Consider the following distributions,

\begin{equation}
\label{eq:equal-marginals}
\begin{split}
& \mathbf{X} = \left \{P(2, 4) = \frac{2}{3}, P(4, 2) = \frac{1}{3} \right \}\\\
& \mathbf{Y} = \left \{P(2, 2) = \frac{1}{3}, P(2, 4) = \frac{1}{3}, P(4, 4) = \frac{1}{3}\right \}.
\end{split}
\end{equation}

Then, $\mathbf{X} \fsd \mathbf{Y}$. However, when looking at the marginal CDFs we see the following:

\begin{equation}
F_{X_1}(2) = F_{Y_1}(2) = \frac{2}{3} \quad F_{X_1}(4) = F_{Y_1}(4) = 1.
\end{equation}

\begin{equation}
F_{X_2}(2) = F_{Y_2}(2) = \frac{1}{3} \quad F_{X_2}(4) = F_{Y_2}(4) = 1.
\end{equation}

And thus $F_{X_1} = F_{Y_1}$ and $F_{X_2} = F_{Y_2}$. As such, $\nexists i \in [d]: X_i \fsd Y_i$.
\end{proof}

Next, we show that contrary to univariate distributions, first-order stochastic dominance does not guarantee a greater or equal utility for all (strictly) increasing utility functions.
\begin{proposition}
\label{prop:fsd-implies-u}
Let $\mathbf{X}$ and $\mathbf{Y}$ be d-dimensional random vectors. Then,
\begin{equation*}
\mathbf{X} \fsde \mathbf{Y} \centernot\implies \forall u \in \mathcal{U}: \E u(\mathbf{X}) \geq \E u(\mathbf{Y}).
\end{equation*}
\end{proposition}
\begin{proof}
Recall the same distributions as shown in \cref{eq:equal-marginals}. Let $u(v_1, v_2) = \ln{(e^0 + e^{v_1})} \cdot \ln{(e^0 + e^{v_2})}$ which is an element of $\mathcal{U}$ and a smooth and strictly increasing approximation of $\max(0, v_1) \cdot \max(0, v_2)$. Then $\E u(\mathbf{X}) \approx 8.55 $, while $\E u(\mathbf{Y}) \approx 9.74$.
\end{proof}

Note that the same proof can be used to show that $\mathbf{X} \fsd \mathbf{Y}$ does not imply a strictly greater expected utility. In fact, the stronger claim of multivariate risk aversion is not enough either as is demonstrated in the following proposition.

\begin{proposition}
\label{prop:fsd-implies-u-mra}
Let $\mathbf{X}$ and $\mathbf{Y}$ be two-dimensional random vectors. Then $\mathbf{X} \fsd \mathbf{Y}$ does not imply in general that $\forall u \in \mathcal{U}$ with 
 $\frac{\partial^2 u(x_1, x_2)}{\partial x_1 \partial x_2} \leq 0: \E u(\mathbf{X}) > \E u(\mathbf{Y})$.
\end{proposition}
\begin{proof}
Recall the same distributions as shown in \cref{eq:equal-marginals}. Let $u(v_1, v_2) = v_1 + v_2$ and note that $u \in \mathcal{U}$ and that  $\frac{\partial^2 u(x_1, x_2)}{\partial x_1 \partial x_2} = 0$. Then $\E u(\mathbf{X}) = 6$ and $\E u(\mathbf{Y}) = 6$.
\end{proof}

Finally, we show that strict first-order stochastic dominance and distributional dominance are equivalent in the case of random vectors with independent random variables.

\begin{proposition}
Let $\mathbf{X}$ and $\mathbf{Y}$ be d-dimensional random vectors such that $P(X_1=x_1, \dotsc, X_d = x_d) = P(X_1=x_1) \dotsm P(X_d = x_d)$ and $P(Y_1=y_1, \dotsc, Y_d = y_d) = P(Y_1=y_1) \dotsm P(Y_d = y_d)$. Then,
\begin{equation*}
    \mathbf{X} \du \mathbf{Y} \iff \mathbf{X} \fsd \mathbf{Y}.
\end{equation*}
\end{proposition}
\begin{proof}
$(\implies)$ First, by definition $\mathbf{X} \du \mathbf{Y} \implies \mathbf{X} \fsde \mathbf{Y}$ and $\exists i \in [d]: X_i \fsd Y_i$. Assume that there does not exist a $\mathbf{v}$ such that $F_{\mathbf{X}}(\mathbf{v}) < F_{\mathbf{Y}}(\mathbf{v})$. Then $\forall \mathbf{v} \in \mathbb{R}^d: F_{\mathbf{X}}(\mathbf{v}) = F_{\mathbf{Y}}(\mathbf{v})$. By our assumption, 
\begin{equation}
\begin{split}
\forall v_i \in \mathbb{R}: & \lim_{\mathbf{v}_{-i} \to \infty} F_{\mathbf{X}_{-i}, X_i}(\mathbf{v}_{-i}, v_i) \\
& = \lim_{\mathbf{v}_{-i} \to \infty} F_{\mathbf{Y}_{-i}, Y_i}(\mathbf{v}_{-i}, v_i),
\end{split}
\end{equation}
leading to a contradiction. It is interesting to note that the implication does not require the random variables in the vectors to be independent.


$(\impliedby)$ Finally, the assumption that all marginals are independent allows us to state that $\forall \mathbf{v} \in \mathbb{R}^d: F_{X_1}(v_1) \dotsm F_{X_d}(v_d) \leq F_{Y_1}(v_1) \dotsm F_{Y_d}(v_d)$ and that $\exists \mathbf{v} \in \mathbb{R}^d: F_{X_1}(v_1) \dotsm F_{X_d}(v_d) < F_{Y_1}(v_1) \dotsm F_{Y_d}(v_d)$. As such, it is necessary that $\exists i \in [d]: X_i \fsd Y_i$.
\end{proof}

\subsection{Relation Between Solution Sets}
In \cref{sec:solution-set} we define a taxonomy between the relevant solution sets in multi-objective decision making. We noted that in general, the Pareto front is neither a subset nor a superset of the convex distributional undominated set. Here, we demonstrate examples of this fact.

\begin{proposition}
\label{prop:pf-subset-cduset}
For any family of policies $\Pi$, it is not true in general that
\begin{equation*}
\pf \subseteq \cduset.
\end{equation*}
\end{proposition}
\begin{proof}
Define the following return distributions:
\begin{equation}
\begin{split}
& \mathbf{Z}^{\pi_1} = \left \{P(1, 5) = 1 \right \}\\
& \mathbf{Z}^{\pi_2} = \left \{P(5, 1) = 1 \right \}\\
& \mathbf{Z}^{\pi_3} = \left \{P(1, 3) = \frac{1}{2}, P(3, 1) = \frac{1}{2} \right \}.
\end{split}
\end{equation}
Then the expected value of these distributions is as follows:
\begin{equation}
\begin{split}
& \mathbf{V}^{\pi_1} = (1, 5)\\
& \mathbf{V}^{\pi_2} = (5, 1)\\
& \mathbf{V}^{\pi_3} = (2, 2).
\end{split}
\end{equation}
Then $\pi_1, \pi_2, \pi_3 \in \pf$. Observe, that
\begin{equation}
\frac{1}{2} \mathbf{Z}^{\pi_1} + \frac{1}{2} \mathbf{Z}^{\pi_2} \du \mathbf{Z}^{\pi_3},
\end{equation}
and therefore $\pi_3 \notin \cduset$.
\end{proof}

\begin{proposition}
\label{prop:cduset-subset-pf}
For any family of policies $\Pi$, it is not true in general that
\begin{equation*}
\cduset \subseteq \pf.
\end{equation*}
\end{proposition}
\begin{proof}
Define the following return distributions:
\begin{equation}
\begin{split}
& \mathbf{Z}^{\pi_1} = \left \{P(2, 5) = 1 \right \}\\
& \mathbf{Z}^{\pi_2} = \left \{P(1, 5) = \frac{1}{2}, P(3, 3) = \frac{1}{2} \right \}.
\end{split}
\end{equation}
Then the expected value of these distributions is as follows:
\begin{equation}
\begin{split}
& \mathbf{V}^{\pi_1} = (2, 5)\\
& \mathbf{V}^{\pi_2} = (2, 4)
\end{split}
\end{equation}
and $\pf = \{\pi_1\}$ while $\cduset = \{\pi_1, \pi_2 \}$.
\end{proof}

\section{A Case Study}
We provide a brief case study to highlight the value of having the additional policies in the distributional undominated set compared to the Pareto front. For this purpose, we define two utility functions in that may be used by real-world decision makers.

The first utility function, shown in \cref{eq:prod}, is the product between the two objectives and is also used in \cref{exmp:utility}. Intuitively, a decision maker with this function aims to strike a balance between both objectives to maximise their utility. We call the decision maker with this utility function Decision Maker 1 (DM1).

\begin{equation}
\label{eq:prod}
    u(v_1, v_2) = v_1 \cdot v_2
\end{equation}

The second utility function, shown in \cref{eq:leontief}, is known as a Leontief utility function and is commonly used in the economic and game theoretic literature to represent the utility function of rational agents \cite{codenotti2007computation}. Here too it is clear that a decision maker needs to take both objectives into account as their utility will be derived from the minimum. We call the decision maker with this utility function Decision Maker 2 (DM2).

\begin{equation}
\label{eq:leontief}
    u(v_1, v_2) = \min(v_1, v_2)
\end{equation}

In \cref{fig:case-study} we show the expected returns of the policies learned using DIMOQ in a sample MOMDP. Conventional approaches in multi-objective reinforcement learning and planning focus their attention on obtaining either the Pareto front, indicated with a black line, or the convex hull, indicated with a grey line. Therefore, only policies in these sets would be presented to the decision makers and all remaining policies would be discarded a priori as they are considered suboptimal.


\begin{figure}[t]
     \centering
     \begin{subfigure}[b]{\columnwidth}
         \centering
         \includegraphics[width=\textwidth]{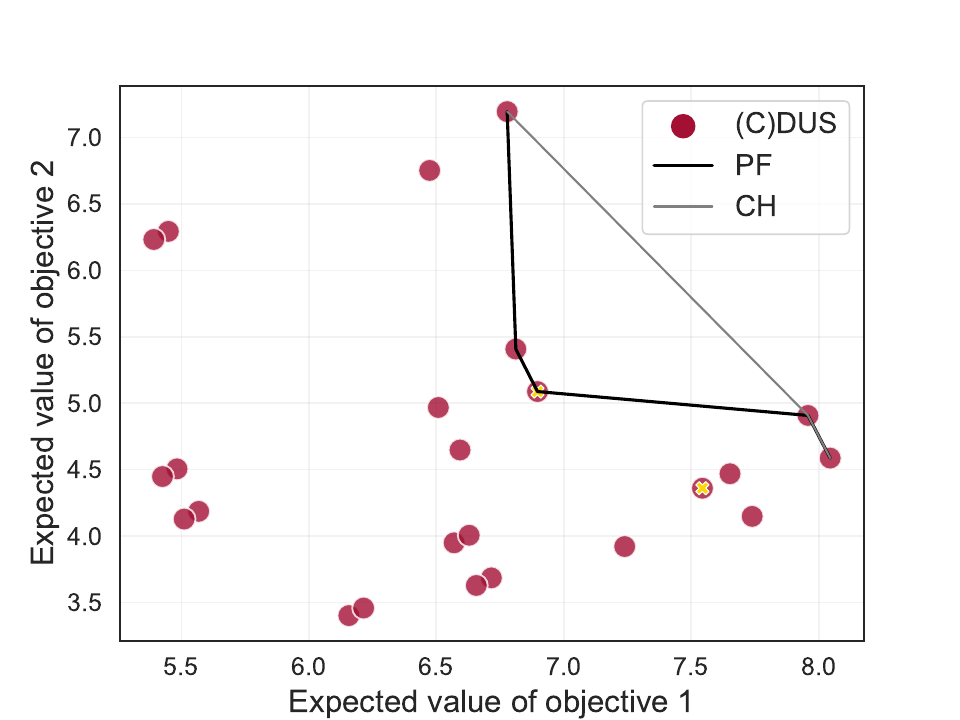}
         \caption{A decision maker with \cref{eq:prod}.}
         \label{fig:cs-1}
     \end{subfigure}
     \hfill
     \begin{subfigure}[b]{\columnwidth}
         \centering
         \includegraphics[width=\textwidth]{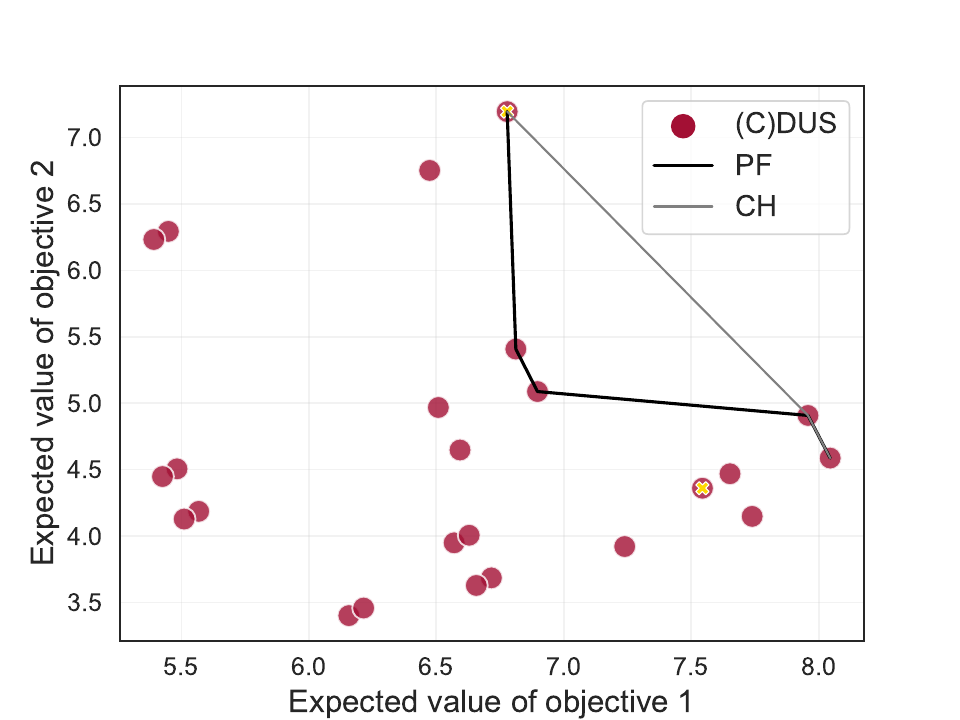}
         \caption{A decision maker with \cref{eq:leontief}.}
         \label{fig:cs-2}
     \end{subfigure}
     \hfill
        \caption{A case study comparing optimal policies in the Pareto front and distributional undominated set for two decision makers.}
        \label{fig:case-study}
\end{figure}

Let us first consider a decision support system that only presents options with a Pareto optimal expected value to the decision maker. In this case, the policy highlighted with a yellow cross on the black line in \cref{fig:cs-1} is optimal for DM1 and in \cref{fig:cs-2} for DM2. These policies lead to an expected utility of $34.87$ for the former and $5.37$ for the latter. Taking a distributional approach, however, better policies can be retrieved. We find that the policy in the Pareto dominated subset of the objective space indicated with a yellow cross is optimal for both decision makers. Concretely, DM1 obtains an expected utility of $49.59$ from this policy while DM2 obtains a utility of $6.49$. Moreover, it is clear that this policy is strictly preferred over all policies contained in the Pareto front or convex hull. Therefore, incorporating a distributional approach when computing a solution set to be used in a decision support system would significantly increase the value a decision maker could get from using the system.

\end{document}